\documentclass[conference,10pt,letterpaper]{IEEEtran}

\usepackage{cite}

\usepackage{stfloats}

\ifCLASSINFOpdf
 \usepackage[pdftex]{graphicx}
\else
\fi
\usepackage[cmex10]{amsmath}
\usepackage{array}
\usepackage{mdwmath}
\usepackage{mdwtab}
\usepackage{eqparbox}
\usepackage{tikz}

\newcommand{\pc}{\mathbf{P}}

\newcommand{\pp}{\mathbb{P}}

\usepackage{amssymb}
\usepackage{amsthm}
\usepackage{subfig}
\usepackage{latexsym}
\usepackage{inputenc}
\usepackage{url}
\usepackage{euler}
\usepackage{amsmath}
\usepackage{amssymb}
\usepackage[inline]{asymptote}
\usepackage{tabularx}

\newtheorem{theorem}{Theorem}
\newtheorem{definition}{Definition}
\newtheorem{proposition}{Proposition}

\begin{document}

\title{Pure Rough Mereology and Counting}
\author{\IEEEauthorblockN{A. Mani}
\IEEEauthorblockA{Department of Pure Mathematics\\
University of Calcutta\\
9/1B, Jatin Bagchi Road\\
Kolkata(Calcutta)-700029, India\\
Email: {a.mani.cms@gmail.com}\\
Homepage: \url{http://www.logicamani.in}}}


\maketitle

\begin{abstract}
The study of mereology (parts and wholes) in the context of formal approaches to vagueness can be approached in a number of ways. In the context of rough sets, mereological concepts with a set-theoretic or valuation based ontology acquire complex and diverse behavior. In this research a general rough set framework called granular operator spaces is extended and the nature of parthood in it is explored from a minimally intrusive point of view. This is used to develop counting strategies that help in classifying the framework. The developed methodologies would be useful for drawing involved conclusions about the nature of data (and validity of assumptions about it) from antichains derived from context. The problem addressed is also about whether counting procedures help in confirming that the approximations involved in formation of data are indeed rough approximations? 
\end{abstract}

\textbf{Keywords}: \begin{small}{Mereology, Parthood, General Granular Operator Spaces, Rough Mereology, Rough Membership Functions, AI, Rough Objects, Granulation, Contamination Problem }\end{small}
\footnote{IEEE-Xplore: WIECON-ECE'2017}

\section{Introduction}

Mereology is the study of parts and wholes and has been studied from philosophical, logical, algebraic, topological and applied perspectives. Almost every major philosophical approach has its approach to the basic questions related to parts and wholes in the context of knowledge, knowledge representation and world view. 

Rough set theory is a formal approach to vagueness and knowledge that involves a wide array of logico-algebraic, computational, mathematical and applied philosophical techniques. The study of mereology in the context of rough sets can be approached in at least two essentially different ways. In the approach aimed at reducing contamination by the present author\cite{AM240,AM3930}, the primary motivation is to avoid intrusion into the data by way of additional assumptions about the data. In the rough membership function based approach \cite{LP2011}, the strategy is to base definitions of parthood on the potential values of the function.   

One of the most general set-theoretic frameworks for granular rough sets has been developed by the present author in \cite{AM6999}. There it is shown by her that the granular framework is ideal for handling antichains formed by mutually distinct objects and related models. Other general set theoretic frameworks avoid granulation and impose more restrictions on approximation. In the present paper, generalized granular operator spaces (GOS) are introduced as a variant of the framework in \cite{AM6999}, the nature of parthood over \textsf{GOS} is investigated, the ontology of meaning associated explored and counting strategies are developed for drawing inferences about the nature of data at hand. The counting strategies refer antichains (derived from parthood) of rough objects, but can be extended to other parthoods in a natural way. These can be used for constructing models and also confirming whether approximations found in practice can possibly be explained from a rough set ontology/view.

\subsection{Background}

An \emph{Information System} $\mathcal{I}$, is a relational system of the form $\mathcal{I}\,=\, \left\langle S,\, \mathbb{A},\, \{V_{a} :\, a\in \mathbb{A}\},\, \{f_{a} :\, a\in \mathbb{A}\}  \right\rangle $
with $S$, $\mathbb{A}$ and $V_{a}$ being respectively sets of \emph{Objects}, \emph{Attributes} and \emph{Values} respectively. Information systems generate various types of relational or relator spaces which in turn relate to approximations of different types and form a substantial part of the problems encountered in general RSTs. 

In classical \textsf{RST}, equivalence relations of the form $R$ are derived by the condition $x,\, y\,\in\, S $ and $B\,\subseteq\, \mathbb{A} $, let $(x,\,y)\,\in\, R $ if and only if $(\forall a\in B)\, \nu(a,\,x)\,=\, \nu (a,\, y)$.  $\left\langle S,\,R \right\rangle $ is then an \emph{approximation space}. On the power set $\wp (S)$, lower and upper approximations of a subset $A\in \wp (S)$ operators, (apart from the usual Boolean operations), are defined as per: $A^l = \bigcup_{[x]\subseteq A} [x]$, $A^{u} = \bigcup_{[x]\cap A\neq \emptyset} [x]$, with $[x]$ being the equivalence class generated by $x\in S$. If $A, B\in \wp (S)$, then $A$ is said to be \emph{roughly included}\label{rin} in $B$, $(A\sqsubseteq B)$ if and only if $A^l \subseteq B^l \,\&$ $A^u\subseteq B^u$. $A$ is \emph{roughly equal} to $B$ ($A\approx B$) if and only if $A\sqsubseteq B$ and $B\sqsubseteq A$ (the classes of $\approx$ are rough objects). 

In rough sets, the objects of interest may be all types of objects, only rough objects of specific type or rough and exact objects of some types and corresponding to these the domains of interest would be the classical domain or rough domain or hybrid versions thereof respectively \cite{AM240}. Boolean algebra with approximation operators forms a classical rough semantics. This fails to deal with the behavior of rough objects alone. The scenario remains true even when $R$ in the approximation space is replaced by arbitrary binary relations. In general, $\wp(S)$ can be replaced by a set with a parthood relation and some approximation operators defined on it as in \cite{AM240}. The associated semantic domain is the classical semantic domain for general RST. The domain of discourse associated with roughly equivalent sets in classical domain is a \emph{rough semantic domain}. Hybrid semantic domains, have also been used in the literature (see \cite{AM240}).

The contamination problem is the problem of reducing confusion among concepts from one semantic domain in another during the process of construction of semantics. The use of numeric functions like rough membership and inclusion functions based on cardinalities of subsets is one source of contamination. The rationale can also be seen in the definition of operations like $\sqcup$ in the definition of pre-rough algebra (for example) that seek to define interaction between rough objects but use classical concepts that do not have any interpretation in the rough semantic domain. Details can be found in \cite{AM3600,AM3930}.

\section{General Granular Operator Spaces}

\begin{definition}
A \emph{General Granular Operator Space} (\textsf{GOS}) $S$ is a structure of the form $S\,=\, \left\langle \underline{S}, \mathcal{G}, l , u, \pc \right\rangle$ with $\underline{S}$ being a set, $\mathcal{G}$ an \emph{admissible granulation}(defined below) over $S$, $l, u$ being operators $:\wp(\underline{S})\longmapsto \wp(\underline{S})$ and $\pc$ being a definable binary generalized transitive predicate (for parthood) on $\wp(\underline{S})$ satisfying the same conditions as in Def.\ref{gos} except for those on admissible granulations (Generalized transitivity can be any proper nontrivial generalization of parthood (see \cite{AM9501}). $\pp$ is  proper parthood (defined via $\pp ab$ iff $\pc ab \,\&\,\neg \pc ba$) and $t$ is a term operation formed from set operations):

\begin{align*}
(\forall x \exists
y_{1},\ldots y_{r}\in \mathcal{G})\, t(y_{1},\,y_{2}, \ldots \,y_{r})=x^{l} \\
\tag{Weak RA, WRA} \mathrm{and}\: (\forall x)\,(\exists
y_{1},\,\ldots\,y_{r}\in \mathcal{G})\,t(y_{1},\,y_{2}, \ldots \,y_{r}) =
x^{u},\\
\tag{Lower Stability, LS}{(\forall y \in
\mathcal{G})(\forall {x\in \wp(\underline{S}) })\, ( \pc yx\,\longrightarrow\, \pc yx^{l}),}\\
\tag{Full Underlap, FU}{(\forall
x,\,y\in\mathcal{G})(\exists
z\in \wp(\underline{S}) )\, \pp xz,\,\&\,\pp yz\,\&\,z^{l} = z^{u} = z,}
\end{align*}
In the granular operator space of \cite{AM6999}, $\pc = \subseteq$, $\pp = \subset$ only in that definition), $\pp$ is  proper parthood (defined via $\pp ab$ iff $\pc ab \,\&\,\neg \pc ba$) and $t$ is a term operation formed from set operations.
\end{definition}

On $\wp(\underline{S})$, if the parthood relation $\pc$ is defined via a formula $\Phi$ as per \begin{equation}\pc ab \text{ if and only if } \Phi(a, b),\end{equation} then the $\Phi$-rough equality would be defined via  $a\approx_\Phi b \text{ if and only if } \pc ab \,\&\pc ba$. In a granular operator space, $\pc = \sqsubset$ is defined by \begin{equation}a \sqsubset b \text{ if and only if } a^l \subseteq b^l \,\&\, a^u \subseteq b^u.\end{equation} The rough equality relation on $\wp(\underline{S})$ is defined via $a\approx b \text{ if and only if } a\sqsubset b  \, \&\,b \sqsubset a$. Regarding the quotient $\underline{S}|\approx$ as a subset of $\wp(\underline{S})$, the order $\Subset$ will be defined as per $\alpha \Subset \beta \text{ if and only if } \Phi(\alpha, \beta)$ Here, $\Phi(\alpha, \beta)$ is an abbreviation for $(\forall a\in \alpha, b\in \beta)\Phi(a, b) $. $\Subset$ will be referred to as the \emph{basic rough order}. 

\begin{definition}
By a \emph{roughly consistent object} will be meant a set of subsets of $\underline{S}$ of the form  $H = \{A ; (\forall B\in H)\,A^l =B^l, A^u = B^u \}$. The set of all roughly consistent objects is partially ordered by the inclusion relation. Relative this maximal roughly consistent objects will be referred to as \emph{rough objects}. By \emph{definite rough objects}, will be meant rough objects of the form $H$ that satisfy 
\begin{equation}(\forall A \in H) \, A^{ll} = A^l \,\&\, A^{uu} = A^{u}. \end{equation} 
Other concepts of rough objects will also be used in this paper.
\end{definition}

\begin{proposition}
When $S$ is a granular operator space, $\Subset$ is a bounded partial order on $\underline{S}|\approx$. More generally it is a bounded quasi order.
\end{proposition}

In quasi or partially ordered sets, sets of mutually incomparable  elements are called \emph{antichains}. Some of the basic properties may be found in \cite{GG1998,koh}. Antichains of rough objects have been used by the present author for forming algebraic models in \cite{AM6999}. In the paper the developed semantics is applicable for a large class of operator based rough sets including specific cases of \textsf{RYS} \cite{AM240} and other less general approaches like \cite{CD3}. In \cite{CD3}, negation like operators are assumed in general and these are not definable operations relative order related operations/relation.

\subsection{Parthood in GOS}

It is necessary to clarify the nature of parthood even in set-theoretic structures like granular operator spaces.
The restriction of the parthood relation to the case when the first argument is a granule is particularly important. The theoretical assumption that \textsf{objects} are determined by their parts, and specifically by granules, may not reasonable when knowledge of the context is evolving. This is because in the situation:
\begin{itemize}
\item {granulation can be confounded by partial nature of information and noise,}
\item {knowledge of all possible granulations may not be possible and the chosen set of granules may not be optimal for handling partial information, and }
\item {the process has strong connections with apriori understanding of the objects in question.}
\end{itemize}

\begin{figure}[bht]
 \includegraphics[width=8.7cm,height=1.0cm]{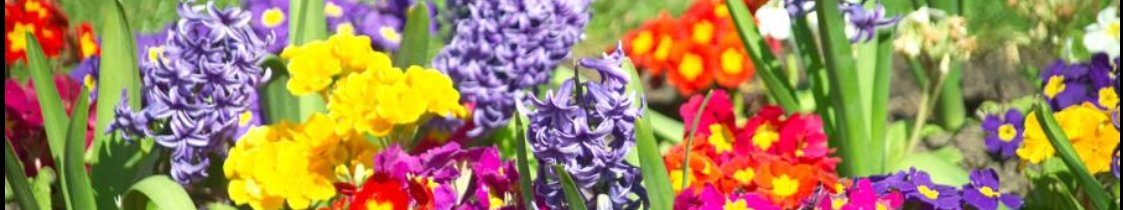}
 \caption{Which Classes and What Parts of Flowers?\small{( FreeFoto)}}
\end{figure}

Parthood can be defined in various ways in GOS. \emph{Not that they are always definable in terms of other predicates/operations, but that many are so definable - it is possible to do rough sets with parthood as a primitive relation}. In picking lavender in a garden as in Fig 1, agents  restrict themselves to classes approximating lavenders for example. Rough inclusion (Sec \ref{rin}) is a simpler example of parthood.  The following are more direct possibilities:
\begin{align*}
 \tag{Very Cautious} \pc ab \longleftrightarrow a^l \subseteq b^l \\
 \tag{Cautious} \pc ab \longleftrightarrow a^l \subseteq b^u \\
 \tag{Lateral} \pc ab \longleftrightarrow a^l \subseteq b^u\setminus b^l \\
 \tag{Possibilist} \pc ab \longleftrightarrow a^u \subseteq b^u \\
 \tag{Ultra Cautious} \pc ab \longleftrightarrow a^u \subseteq b^l \\
 \tag{Lateral+} \pc ab \longleftrightarrow a^u \subseteq b^u\setminus b^l \\
 \tag{Bilateral} \pc ab \longleftrightarrow a^u\setminus a^l \subseteq b^u \setminus b^l \\
 \tag{Lateral++} \pc ab \longleftrightarrow a^u\setminus a^l \subseteq b^l \\
 \tag{G-Simple} \pc ab \longleftrightarrow (\forall g\in \mathcal{G})(g\in x \longrightarrow g\in y)
 \end{align*}

All of these except for lateral++ are transitive concepts of parthoods that make sense in contexts as per availability and nature of information. \textsf{Very cautious} parthood makes sense in contexts in which cost of misclassification is high or the association between properties and objects is confounded by the lack of clarity in the possible set of properties. \textsf{G-Simple} is a version that refers granules alone and avoids references to approximations. 

\begin{theorem}
Lateral++ parthood is not a transitive and reflexive relation but is a strictly confluent relation (as it satisfies $\pc ab \& \pc a c \longrightarrow (\exists e) \pc be \& \pc ce $). In all other cases, in the above, $\pc$ is reflexive and transitive. Antisymmetry need not hold in general.
\end{theorem}

\begin{proof}
If $\pc$ is a bilateral relation as defined above, then $\pc aa$ because $a^u \setminus a^l = a^u \setminus a^l$.
If $\pc ab $ and $\pc bc$, then $a^u \setminus a^l \subseteq b^u \setminus b^l$ and $b^u \setminus b^l \subseteq c^u \setminus c^l$ so $a^u \subseteq a^l = c^u \setminus c^l$ that is $\pc ac$.
Similarly other cases can be verified. Counterexamples for lateral++ parthood are easy to construct.
\end{proof}

\section{Counting for Semantics}

If it is possible to count (in a generalized sense) a collection of objects in accordance with rules dependent on the relative types of objects then these may be useful for discovering new inferences about the collection - this idea has been used in \cite{AM240} in the context of the contamination reduction approach for actually deducing algebraic semantics for classical rough sets. The \emph{History based primitive counting} (\textsf{HPC}) method \cite{AM240} is a way of counting discernible and indiscernible objects in which full memory of objects being counted is retained over temporal progression. This does not lead to antichains, but the modification proposed in this section permits as much.

The collection being counted is taken to be $\{x_{1},\,x_{2},\,\ldots,\,x_{k},\,\ldots ,\,\}$ for simplicity and $R$ is a general indiscernibility over it. The relation with all preceding steps of counting is taken into account and when $R$ is symmetric, the following is the \textsf{HPC} algorithm: 

\begin{align*}
\text{Assign } f(x_{1}) =  1_{1} = s^{0}(1_{1}) \\
\text{If } f(x_{i}) = s^{r}(1_{j}) \,\&\, (x_{i},x_{i+1}) \in  R,\\
\text{ then assign } f(x_{i+1} = 1_{j+1} \\
\text{If } f(x_{i}) = s^{r}(1_{j}) \,\&\, (\forall k < i+1 ) (x_{k},x_{i+1}) \notin R,\\
\text{then let } f(x_{i+1}) = s^{r+1}(1_{j}) 
\end{align*}

For this section, it will be assumed that $S$ is a granular operator space, $\mathbb{S} \subseteq \wp(\underline{S})$, $\mathbb{R} \subset \mathbb{S}$ is the set of rough objects (some sense),
$\mathbb{C} \subseteq \mathbb{S}$ is the set of crisp objects and there exists a map $\varphi : \mathbb{R} \longmapsto \mathbb{C}^{2}$ satisfying $ (\forall x\in \mathbb{R})(\exists a, b\in \mathbb{C}) \varphi (x) = (a, b) \,\& \, a\subset b$,  $\# (\mathbb{S}) = n < \infty$, $\# (\mathbb{C}) = k$, $\mathbb{R\cap C} = \emptyset$ and $ \# (\mathbb{R}) = n-k < n$.

\begin{proposition}
Even though it is not required that $\varphi (x) = (a, b) \, \&\, x^l = a \, \&\, x^u = b$, $\mathbb{R}$ must be representable by a finite subset $\mathbb{K} \subseteq \mathbb{C}^2 \setminus \mathbb{C}$.   
\end{proposition}

\subsection*{Primitive Counting on Antichains(PCA) }

The following conveys the basic idea of primitive counting on antichains (better algorithms may be possible). Every parthood $\pc$ can be associated with $\pc$-antichain generated by the relation $R$ defined by $R ab$ if and only if $\neg \pc a b \,\&\,\neg \pc b a$.
\begin{itemize}
\item {Assign $x_1$ to category $C_1$ and set $f(x_1) = 1_1$.}
\item {If $\neg Rx_1x_2$ then assign $x_2$ to category $C_1$ and set $f(x_2) = 2_1$ else assign $x_2$ to a new category $C_2$ and set $f(x_2) = 1_2$.  }
\item {If $f(x_2) = 2_1$, $\neg Rx_1 x_3$ and $\neg R x_2 x_3$ then assign $x_3$ to category $C_1$ and set $f(x_3) = 3_1$. If $f(x_2) = 2_1$ and ($ Rx_1 x_3$ or $R x_2 x_3$) then assign $x_3$ to category $C_3$ and set $f(x_3) = 1_3$. If $f(x_2) = 1_2$ and $\neg Rx_1 x_3$ then assign $x_3$ to category $C_1$ and set $f(x_3) = 2_1$. If $f(x_2) = 1_2$ and ($ Rx_1 x_3$ and $\neg Rx_2 x_3$ then assign $x_3$ to category $C_2$ and set $f(x_3) = 2_2$ }
\item {Proceed Recursively under the following conditions:}
\item {No two distinct elements $a,\, b$ of a category $C_i$ satisfy $Rab$ }
\item {For distinct $i, j$, $(\forall a\in C_i)(\exists b\in C_j)\, Rab$}
\end{itemize}

\begin{theorem}
In the above context all of the following hold: 
\begin{itemize}
\item {Objects of category $C_1$ form a maximal antichain. }
\item {The objects in category $C_i$ would be enumerated by sequences of the form ($Q_i = \text{Card}(C_i)$)
$1_i, 2_i , \ldots , Q_i$ and $\sum Q_i = n$.}
\item {Objects of each category forms an antichain and each $x_i$ belongs to exactly one of the categories.}
\end{itemize}
\end{theorem}
\begin{proof}
Most of the proof follows from the construction.  All objects of the category $C_1$ are mutually compared and objects not in $C_1$ are $R$-related to at least one object in $C_1$. So $C_1$ is a maximal antichain.  
\end{proof}

\subsection*{History Aware Primitive Counting on Antichains (HPCA) }

In this variation of \textsf{PCA}, objects will be permitted to belong to multiple categories - with the aspect being reflected in the end result of the counting process. This allows for a enumeration for a maximal antichain decomposition. The steps for this method of counting are as follows:

\begin{itemize}
\item {Assign $x_1$ to category $C_1$ and set $f(x_1) = 1_1$.}
\item {If $\neg Rx_1x_2$ then assign $x_2$ to category $C_1$ and set $f(x_2) = 2_1$ else assign $x_2$ to a new category $C_2$ and set $f(x_2) = T_2$.  }
\item {If $f(x_2) = 2_1$, $\neg Rx_1 x_3$ and $\neg R x_2 x_3$ then assign $x_3$ to category $C_1$ and set $f(x_3) = 3_1$. If $f(x_2) = 2_1$ and ($ Rx_1 x_3$ or $R x_2 x_3$) then assign $x_3$ to category $C_3$ and set $f(x_3) = T_3$. }
\item {For $j>2$, if $f(x_j) = k_1$ for $k\leq j$, $\neg Rx x_{j+1}$ for all $x \in C_1$ ($C_1$ at this step) then assign $x_{j+1}$ to category $C_1$ and set $f(x_{j+1}) = {k+1}_1$. If $f(x_j) = k_1$ and ($ Rx x_{j+1}$ for at least an $x\in C_1$ at this stage) then assign $x_{j+1}$ to category $C_{j-k+1}$ and set $f(x_{j+1}) = T_{j+1}$.}
\item {Continue till all objects have been covered for the construction of $C_1$. The only function of $T_j$ is in locating the start point for the next step. }
\item {Stopping Condition: $\cup C_i = S$ (at this stage) and for $i\neq j$ $C_i \nsubseteq C_j$ or all start points are exhausted. }
\item {Find the argument for $f$ for which $j$ in $T_j$ is a minimum}
\item {Assign $x_j$ to the category $C_2$ and set $f(x_j) = 1_2$. }
\item {Proceed as for $x_1$, but over the set \begin{equation}S^{(j)} = \{x_j, x_{j+1}, \ldots , x_n, x_1, x_2, \ldots, x_{j-1} \}\end{equation} to derive all the elements of $C_2$. Continue till stopping condition is satisfied. }
\end{itemize}

\begin{theorem}
\textsf{HPCA} may possibly yield a decomposition of $S$ into maximal antichains $C_1, C_2, \ldots ,C_q $ or just a collection of maximal antichains $C_1, C_2, \ldots ,C_q $ satisfying \begin{equation}\bigcup C_i \subseteq S.\end{equation}
\end{theorem}

\begin{proof}
The order structure initially assumed on $S$, restricts the maximal antichains that may be found by the \textsf{HPCA} algorithm. 

It is possible that the order in which the discernible objects are found restricts the discernibility of objects found subsequently. So many maximal antichains are bound to be excluded. Counter examples are easy to construct.
\end{proof}

The above proof motivates the following definition:
\begin{definition}
A total order $<$ on $S$ will be said to be \emph{HPCA coherent} if and only if the \textsf{HPCA} procedure generates a set of maximal antichains $\{C_i : i = 1, 2, \ldots ,q\}$ such that $\cup C_i = S$. 
\end{definition}

\subsection*{Full History Based Counting on Antichains (FHCA)}

In this method the steps shall be the same as for \textsf{HPCA} if the order is HPCA coherent. If not, then

\begin{itemize}
\item {Store the maximal antichain $C_1$ as $A_1$. }
\item {Permute the order on $S$ by the permutation $\sigma_1$. }
\item {Compute $C_1$ as per the HPCA algorithm for the new order and store it as $A_2$ - the elements of $A_2$ being numbered as for $C_2$, that is in the form $1_2, 2_2, \ldots , Q_2$. }
\item {Stop if $\cup A_i = S$}
\item {Else permute the order on $S$ by a permutation $\sigma_2$ and repeat the above steps to obtain $A_3$ - the elements of $A_3$ being numbered as for $C_3$, that is in the form $1_3, 2_3, \ldots , Q_3$.}
\item {Proceed till stopping criteria is satisfied.}
\end{itemize}

\begin{theorem}
Let the set of maximal antichains obtained by the FHCA method be $\mathbb{A}$. $\mathbb{A}$ need not be associated with a HPCA coherent order $<$. 
\end{theorem}

\subsection{Applications}

The novel application scenarios to which the developed methodology apply qualify as parts of the inverse problem \cite{AM240}. Based on some understanding of properties and their connection with objects (this need not be clearly understood beforehand), suppose an agent specifies approximations and indiscernibilities of objects. In the usual way of applying rough sets, approximations are discovered from concrete information systems/tables. But here the origin can be very dense.

With the approximations and indiscernibilities of objects, it would be certainly possible to count the objects in various orders in the senses mentioned in this section. Using this it would be possible to predict whether the information has been derived from a rough perspective or not. Algebraic semantics may also be derivable if admissible. More details relating to this would be part of a forthcoming paper.

If rough sets are interpreted from the perspective of knowledge (see \cite{AM909,PPM2} and references therein), definite objects represent exact concepts, while others are not exact. Non definite objects may be approximated by \emph{lower definite objects} that correspond to the part of knowledge that is definitely contained in the object.   
For this to be valid it is necessary that for the object $A$ in question, $A^{ll}= A^l$ and $A^{lu}=A^l $ hold.
Similarly when $A$ is approximated by its upper approximation $A^u$, the latter corresponds to the possible concepts that may correspond to $A$. For this interpretation to be valid, it is again necessary that $A^{uu} = A^u$ hold. In general, the approximations $A^l, A^u$ may not correspond to exact knowledge and many generalized concepts of exact knowledge may be definable/relevant. In fact, \textsf{GOS} permits encoding knowledge that can be successively refined on the \emph{possible knowledge} aspect for every choice of parthood. The present paper also motivates the problem of characterization of the correspondence between the general counts and rough ontology and the structures analogous to antichains associated with generalized transitive parthoods. Specifically, the algebras corresponding to \textsf{FHCA} and \textsf{HPCA} are of interest.

\bibliographystyle{IEEEtran}
\bibliography{../bib/biblioam2016xx.bib}
\end{document}